\documentclass[sigconf]{acmart}

\usepackage{centernot}

\theoremstyle{acmdefinition}
\newtheorem{observation}[theorem]{Observation}

\begin{document}

\title{Discovering Evolutionary Stepping Stones\\ through Behavior Domination}

\author{Elliot Meyerson}
\affiliation{
	\institution{UT Austin; Sentient Technologies, Inc.}
}
\email{ekm@cs.utexas.edu} 

\author{Risto Miikkulainen}
\affiliation{
	\institution{UT Austin; Sentient Technologies, Inc.}
}
\email{risto@cs.utexas.edu} 

\acmConference[GECCO '17]{the Genetic and Evolutionary Computation Conference 2017}{July 15--19, 2017}{Berlin, Germany}
\acmYear{2017}
\copyrightyear{2017}

\begin{abstract}
\emph{Behavior domination} is proposed as a tool for understanding and harnessing the power of evolutionary systems to discover and exploit useful stepping stones. 
Novelty search has shown promise in overcoming deception by collecting diverse stepping stones, and several algorithms have been proposed that combine novelty with a more traditional fitness measure to refocus search and help novelty search scale to more complex domains. 
However, combinations of novelty and fitness do not necessarily preserve the stepping stone discovery that novelty search affords.
In several existing methods, competition between solutions can lead to an unintended loss of diversity.
Behavior domination defines a class of algorithms that avoid this problem, while inheriting theoretical guarantees from multiobjective optimization.
Several existing algorithms are shown to be in this class, and a new algorithm is introduced based on fast non-dominated sorting.
Experimental results show that this algorithm outperforms existing approaches in domains that contain useful stepping stones, and its advantage is sustained with scale.
The conclusion is that behavior domination can help illuminate the complex dynamics of behavior-driven search, and can thus lead to the design of more scalable and robust algorithms.
\end{abstract}




\title{Discovering Evolutionary Stepping Stones\\ through Behavior Domination}
\maketitle

\section{Introduction}
\begin{figure}
\includegraphics[width=0.9\columnwidth]{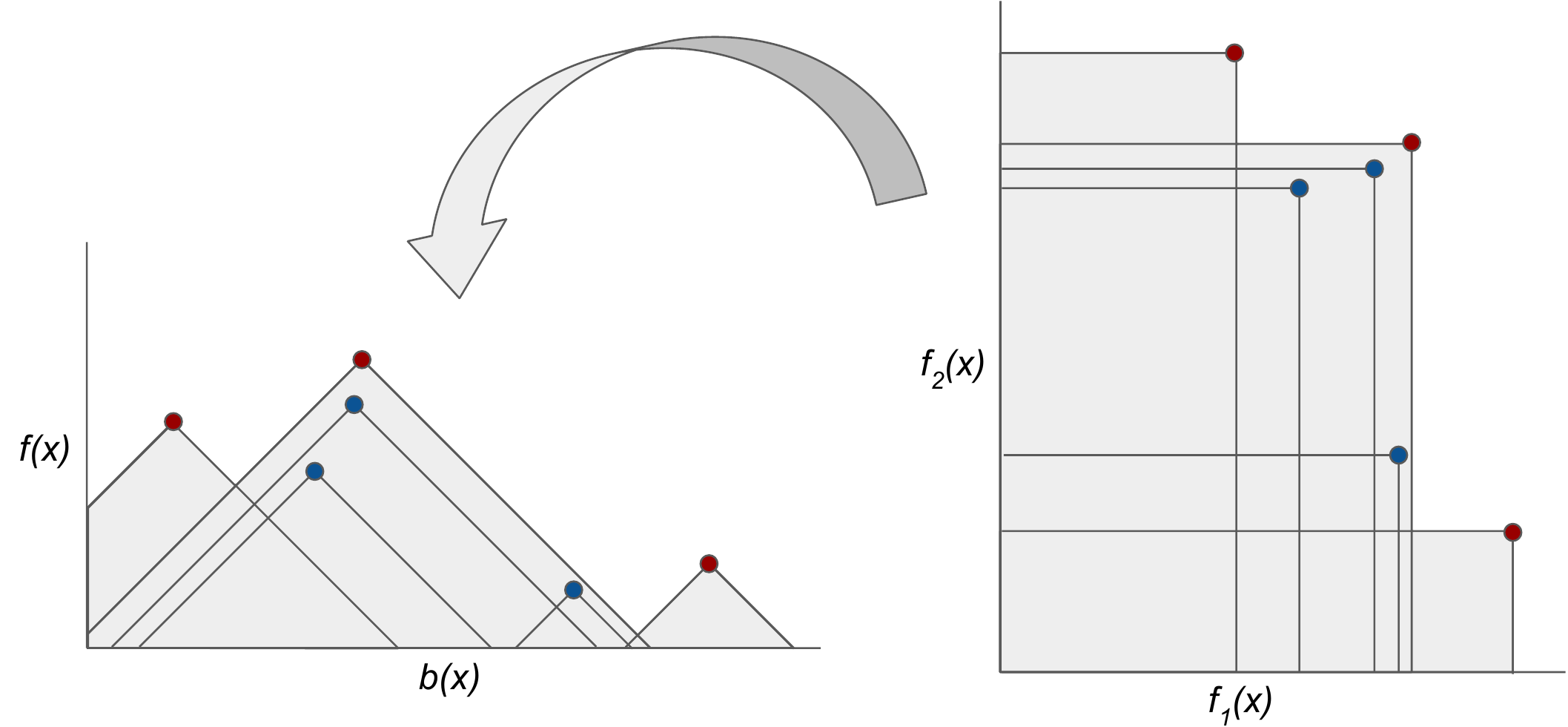}
\vspace{-5pt}
\caption{The non-dominated front of behavior domination can be viewed as a rotation of the Pareto front. Guarantees from multiobjective optimization can then be applied.}
\label{FigRotation}
\end{figure}
The ability to discover and exploit stepping stones is a hallmark of evolutionary systems. 
Evolutionary algorithms driven by a single fitness objective are often victims of \emph{deception}: they converge to small areas of the search space, missing available stepping stones.
Novelty search \cite{lehman08, lehman11a} is an increasingly popular paradigm that overcomes deception by ranking solutions based on how different they are from others.
Novelty is computed in the space of \emph{behaviors}, i.e., vectors containing semantic information about \emph{how} a solution achieves its performance when it is evaluated.
In a collection of solutions with sufficiently diverse behaviors, some solutions will be useful stepping stones.
However, with a large space of possible behaviors, novelty search can become increasingly unfocused, spending most of its resources in regions that will never lead to promising solutions.
Recently, several approaches have been proposed to combine novelty with a more traditional fitness objective \cite{mouret15a, mouret09, gomes15, gomez09, pugh15} to reorient search towards fitness as it explores the behavior space.
These approaches have helped scale novelty search to more complex environments, including an array of control \cite{mouret15, mouret12, bowren16} and content generation \cite{lehman11b, liapis13, preuss14, lehman12, nguyen15, nguyen16, lehman16} domains.

This paper shows that, aside from focusing search overall, the addition of fitness can also be used to focus search on discovering useful stepping stones.
The assumption is that the most likely stepping stones occur at local optima along some dimensions of the behavior space.
Competition in several existing algorithms inhibits the discovery and maintenance of such stepping stones, resulting in ``spooky action at a distance'', when a small search step in one part of the space causes a novel solution to be lost in another part.
Based on the notion of \emph{behavior domination}, a class of algorithms is defined in this paper as a framework for understanding the dynamics of behavior-driven search and developing algorithms that avoid such problems.
Intuitively, behavior domination means that a solution exerts a negative effect on the ranking of every weaker solution, and this effect increases as their difference in fitness increases \emph{and} as the distance between their behaviors decreases.
Behavior domination algorithms include several existing algorithms, and the definition makes it possible to transfer theoretical guarantees from multiobjective optimization; the non-dominated front induced by behavior domination can be viewed (Figure~\ref{FigRotation}) as a rotation of a Pareto front.
Within this framework, a new algorithm is developed that uses fast non-dominated sorting \cite{deb02}.
Experimental results show that this algorithm outperforms existing approaches in domains that contain useful stepping stones, and its advantage is sustained with scale.
The conclusion is that behavior domination can help illuminate the complex dynamics of behavior-driven search, and can thus lead to the design of more scalable and robust algorithms.

\section{Behavior-driven Ranking} \label{SecBehaviorDriven}

Behavior-driven algorithms are a class of evolutionary algorithms that are guided by information about \emph{how} a solution achieves its performance during evaluation.
The core defining component of such an algorithm is the ranking procedure it uses to order solutions for selection or replacement. 
This section reviews background for behavior-driven search, first defining some useful terms, and then describing examples of popular behavior-driven algorithms. 

\subsection{Behavior and Behavior Characterization} \label{SubSecBehavior}

Behavior-driven algorithms use a notion of solution behavior to induce a meaningful distance metric between solutions and to facilitate the drive towards novelty and diversity. 
For example, in a robot control domain, a solution's behavior may be some function of the robot's trajectory \cite{gomez09, gomes13, mouret12}, whereas in an image generation domain, it may be the result of applying some deep features to the image \cite{liapis13, nguyen15, nguyen16, lehman16}. 
The following definitions of behavior, behavior characterization, behavior space, and behavior distance are fairly universal in the literature, though often not explicitly defined.

\begin{definition}
A \emph{behavior} of solution $x$ in environment $E$ is a vector $b_x$ resulting from the evaluation of $x$ in $E$.
\end{definition}

\begin{definition}
A \emph{behavior characterization} $b(x)$ for an environment $E$ is a (possibly stochastic) function mapping any solution $x$ to its behavior $b_x$, given the evaluation of $x$ in $E$. 
\end{definition}

By definition, the behavior characterization can be any function mapping solutions to vectors. In practice, the behavior characterization is usually designed to \emph{align} with a fitness measure or notion of interestingness in the evaluation environment \cite{pugh15}. For example, in a maze navigation task, the final position of a robot aligns more with solving the task than its final orientation. In other words, the behavior characterization is designed to capture a space whose exploration is expected to have practical benefits. 

\begin{definition}
The \emph{behavior space} of a behavior characterization $b$ is the co-domain of $b$.
\end{definition}

The exploration of the behavior space by a search algorithm is facilitated by a function giving the distance between two solutions as a function of their behavior.

\begin{definition}
A \emph{behavior distance} is a metric $d(b(x), b(y))$. 
\end{definition}

In pure novelty search, the behavior of a solution is the only information returned from evaluation that is used in the ranking system. This is in contrast to traditional evolutionary algorithms, which use only a single scalar fitness value $f_x$ computed from a scalar fitness function $f(x)$. In general, a behavior-driven algorithm can take advantage of both behavior and fitness when ranking solutions.

\subsection{Existing Behavior-driven Algorithms} \label{SubSecExistingBDAs}

The following are some of the most popular schemes for behavior-driven algorithms. As extensions to the pure novelty search paradigm, several recent algorithms use both behavior and fitness information in ranking, trying to navigate the trade-off between the pressures towards novelty and diversity, and the pressure to maximize. Although more exist that are not covered here, these below should give a sense of the behavior-driven algorithm design space. (See \cite{mouret12, pugh15,gomes15} for previous reviews of these algorithms.)

\subsubsection{Novelty search (NS) \cite{lehman08, lehman11a}}
Each solution is ranked based on a single \emph{novelty} function $n$, giving the average distance of its behavior to the $k$ nearest behaviors of other solutions in the population and an archive of past solutions accumulated throughout search. More specifically,
$$ n(x) = \frac{1}{k}\sum_{i = 1}^{k} d(b(x), b(y_i))$$
where $y_i$ is the $i^{th}$ nearest neighbor of $x$ in the behavior space. 
The prevalent method of building the archive, and the method used in this paper, is to add each solution to the archive with a fixed probability $p_{add}$ \cite{lehman10b, gomes15}, in which case the archive represents a sampling from the distribution of areas visited so far.
Novelty search captures the idea that more complex and interesting solutions lie away from the visited areas of the behavior space.

\subsubsection{Linear scalarization of novelty and fitness (LSNF) \cite{cuccu11,gomes15}}
An intuitive method of combining novelty and fitness is to rank a solution based on linear scalarization of its fitness and novelty:
$$ \text{score}(x) = (1 - p) \cdot \frac{f(x) - f_{min}}{f_{max} - f_{min}} + p \cdot \frac{n(x) - n_{min}}{n_{max} - n_{min}} n(x). $$
The fitness and novelty scores here are normalized to compensate for differences in scale at every iteration. $f_{min}$, $f_{max}$, $n_{min}$, and $n_{max}$ are the minimum and maximum fitness and novelty scores in the current population. The parameter $p$ controls the trade-off of fitness vs. novelty. LSNF with $p = 0.5$ has been shown to be robust across domains \cite{gomes15}, and that is the version considered here.

\subsubsection{NSGA-II with novelty and fitness objectives (NSGA-NF) \cite{mouret09, mouret12}}
Another approach is to use novelty and fitness as two objectives within NSGA-II \cite{deb02}, the popular multiobjective framework. Often the novelty score in this approach is \emph{behavioral diversity}, which is a special case of novelty, where $k$ is the population size and there is no archive.
This approach has been shown to improve performance on many tasks, especially those in evolutionary robotics, where some constant diversity is useful to avoid local optima. 

\subsubsection{Novelty search with local competition (NSLC) \cite{lehman11b, pugh15}}
Novelty search with local competition also uses an NSGA-II ranking system, but instead of using a raw fitness objective alongside the novelty objective, it uses a relative fitness score: a solution's rank in fitness among its $k$ nearest neighbors. This enables the suitable exploration of diverse niches in the behavior space with different orders of magnitude of fitness. Lower fit niches are not outpaced and forgotten by having too much of the search's resources comitted to the globally most fit regions. NSLC has yielded particularly promising results in content generation domains, such as generating virtual creatures and images \cite{lehman11b, nguyen15}. 

\subsubsection{MAP-elites \cite{mouret15, mouret15a}}
In MAP-elites, the behavior space is broken up into a set of bins, such that each behavior is mapped to a bin. For each bin, the solution with highest fitness whose behavior falls into that bin is kept. The population at any point thus consists of the most fit (elite) solution from each bin for which a behavior has been found. Because MAP-elites keeps an elite from all visited bins in the behavior space, at any point the population displays a map of the levels of fitness achievable throughout the space. So, along with being a method for generating high-quality diverse solutions, MAP-elites is a useful tool for visualization in understanding how the behavior space and fitness landscape relate.

\subsubsection{Fitness-based search}
It is worth including fitness-based search, the standard approach to evolutionary search, as the trivial example. In fitness-based search, solutions are ranked based on a single fitness value. Any additionally available behavior information is ignored.\\

\noindent The proliferation of recently introduced behavior-driven methods gives a strong indication that novelty alone is not generally sufficient for tackling complex domains.
The methods reviewed above each have intriguing definitions that suggest they would be a good option for particular kinds of problems.
However, unforeseen dynamics can emerge from the interaction between novelty and fitness, which can be difficult to disentangle.
The next section sheds some light on these issues, resulting in the characterization of these existing algorithms, and the development of a new approach.

\section{Behavior Domination Algorithms} \label{SecBDAs}
The goal is to maintain the power of novelty search to discover stepping stones, while adding a fitness drive to focus search.
Novelty search has demonstrated that a sufficiently diverse collection of solutions most likely contains useful stepping stones for solving the problem at hand.
When adding fitness to focus search, the presumption is that the most useful stepping stones will be local optima along some dimensions of the behavior space.
As pure fitness-based search maintains the most fit solutions, and pure novelty search maintains the most novel solutions, a method that combines the two should maintain the most promising set of stepping stones discovered so far, and the quality of this set should improve over time.
Section~\ref{SubSecSpooky} discusses the presence of ``spooky action at a distance'' in several existing algorithms, which inhibits their ability to preserve useful stepping stones.
Section~\ref{SubSecBDF} presents a formalization of behavior domination, which defines a sub-class of behavior-driven algorithms that can avoid this pitfall and guarantee monotonic improvement of collected stepping stones.
Section~\ref{SubSecExistingBDMAs} shows that several existing behavior-driven algorithms are in this sub-class.
Section~\ref{SecNewBDMA} uses behavior domination to develop a new algorithm based on fast non-dominated sorting. 

\subsection{``Spooky Action at a Distance'' for Behavior-driven Search} \label{SubSecSpooky}

When novelty and fitness are combined, the interaction between these two drives can have unintended consequences. 
The stepping stone discovery ability of novelty search may not necessarily be preserved.
For example, if a small change in behavior of one solution has a fatal effect on a distant isolated solution on the other edge of the explored behavior space, then a valuable stepping stone may be lost.
The algorithm has taken one small step forward, but one large step back.
This unsettling effect is an instance of ``spooky action at a distance'' for behavior-driven search.
More specifically, spooky action at a distance occurs when a ranking decision based on a \emph{local} increase in novelty results in a \emph{global} decrease of novelty.
Here, global novelty is defined by two measures: GNP, the maximum behavior distance between any pair of solutions in the population; and GNT, the total behavior distance between all pairs of solutions.

It turns out several existing behavior-driven algorithms support spooky action at a distance. The following example is for a one-dimensional behavior space. Consider a population $P = \{x_0, x_1, x_2, x_3\}$, and an empty archive, where $b(x_0) = b_o$, $f(x_0) = f_o$, $b(x_1) = b_o + 10$, $f(x_1) = f_o + 11$, $b(x_2) = b_o + 11$, $f(x_2) = f_o + 10$, $b(x_3) = b_o + 21$, and $f(x_3) = f_o$. Now, consider an identical setup but with $P' = \{x_0, x_1, x_2, x_4\}$, where $b(x_4) = b_o + 22$, and $f(x_4) = f_o$ (Figure~\ref{FigSpooky}).
\begin{figure}
\includegraphics[width=0.95\columnwidth]{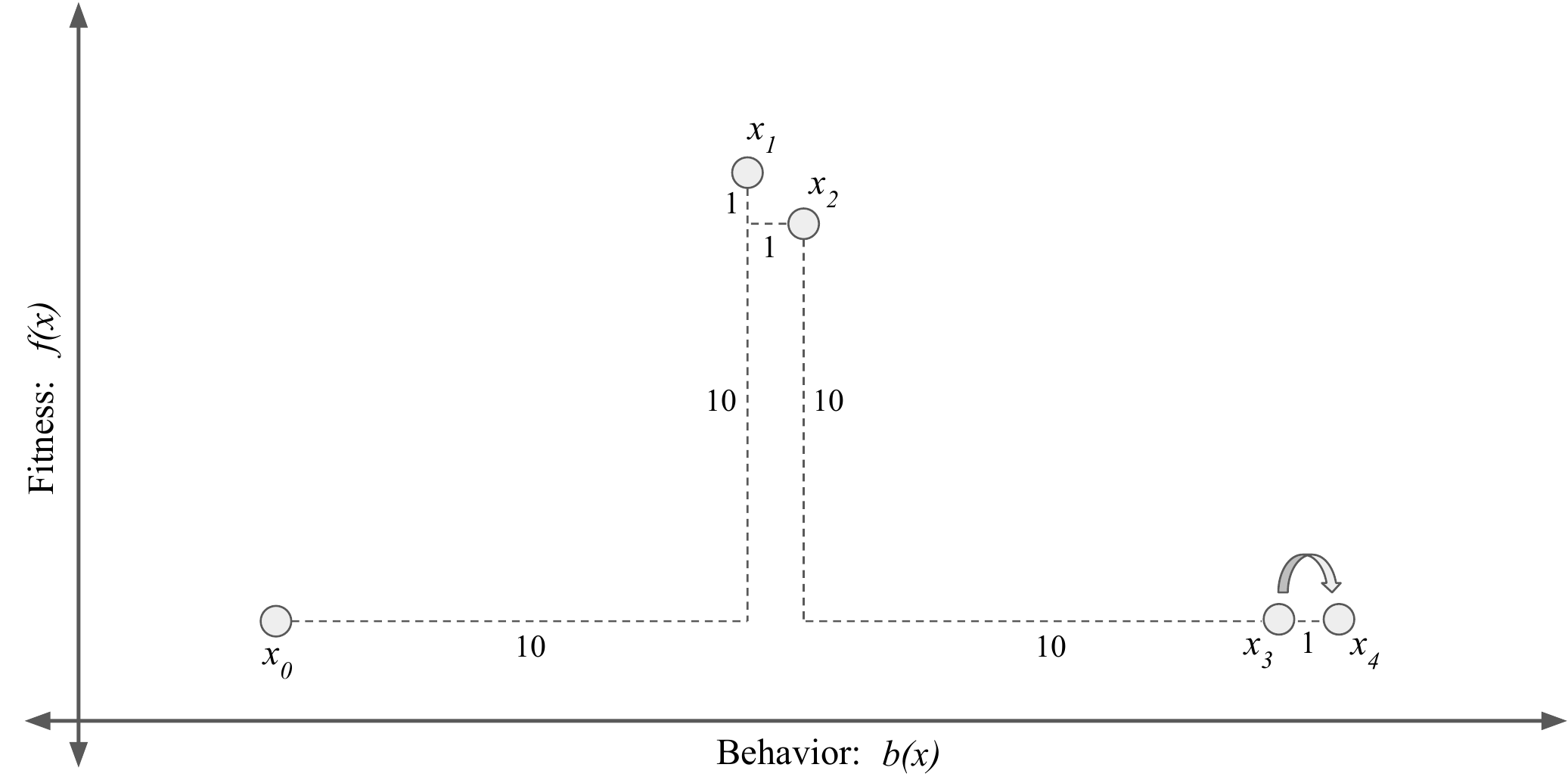}
\caption{\emph{(spooky action at a distance)} Consider populations $P = \{x_0, x_1, x_2, x_3\}$ and $P' = \{x_0, x_1, x_2, x_4\}$, in which one solution must be selected for deletion. Suppose $k = 2$, and the archive is empty. With population $P$, LSNF, NSGA-NF, and NSLC all delete $x_2$. However, with population $P'$, they all delete $x_0$. The small \emph{local} increase in novelty from $x_3$ to $x_4$ thus causes a \emph{global} decrease in novelty (Section~\ref{SubSecSpooky}).}
\label{FigSpooky}
\end{figure}
Suppose an algorithm $A$ must delete one solution, and $A$ deletes $x_2$ with population $P$, but $A$ deletes $x_0$ with population $P'$. This change must be caused by the move of $x_3$ to $x_4$. $P$ with $x_2$ deleted has global novelty $\mbox{GNP}(P) = 21$ and $\mbox{GNT}(P) = 41$. However, $P'$ with $x_0$ deleted has global novelty $\mbox{GNP}(P') = 12$ and $\mbox{GNT}(P') = 24$. Thus, $A$ demonstrates spooky action at a distance.

Suppose $k = 2$. Then given $P$, $n(x_0) = 21/2$, $n(x_1) = 11/2$, $n(x_2) = 11/2$, and $n(x_3) = 21/2$. Given $P'$, $n(x_0) = 21/2$, $n(x_1) = 11/2$, $n(x_2) = 12/2$, and $n(x_4) = 23/2$. The next three observations show spooky action at a distance for LSNF, NSGA-NF, and NSLC.

\begin{observation}[Spookiness of LSNF] \label{ObsSpookyLSNF}
With $P$, $\mbox{score}(x_0) = 0 + 1$, $\mbox{score}(x_1) = 1 + 0$, $\mbox{score}(x_2) = 10/11 + 0$, and $\mbox{score}(x_3) = 0 + 1$ $\implies$ $x_2$ is deleted. With $P'$, $\mbox{score}(x_0) = 0 + 10/12$, $\mbox{score}(x_1) = 1 + 0$, $\mbox{score}(x_2) = 10/11 + 1/12$, and $\mbox{score}(x_4) = 0 + 1$ $\implies$ $x_0$ is deleted.
\end{observation}

\begin{observation}[Spookiness of NSGA-NF] \label{ObsSpookyNSGANF}
With $P$, $x_1$ dominates $x_2$, while all other solutions are non-dominated $\implies$ $x_2$ is deleted. With $P'$, $x_2$ is no longer dominated, but $x_4$ now dominates $x_0$ $\implies$ $x_0$ is deleted.
\end{observation}

\begin{observation}[Spookiness of NSLC] \label{ObsSpookyNSLC}
With $P$, the local competition scores of $x_0, x_1, x_2, x_3$ are $0, 2, 1, 0$, resp. So, $x_1$ dominates $x_2$, while all other solutions are non-dominated $\implies$ $x_2$ is deleted.  With $P'$, the local competition scores of $x_0, x_1, x_2, x_4$ are again $0, 2, 1, 0$, resp. So, as in Observation~\ref{ObsSpookyNSGANF}, $x_2$ is no longer dominated, but $x_4$ now dominates $x_0$ $\implies$ $x_0$ is deleted.
\end{observation}

With problems such as ``spooky action at a distance'' in mind, the next section introduces a notion of behavior domination from which algorithms can be developed that avoid these issues.

\subsection{Ranking by Behavior Domination} \label{SubSecBDF}

A practical unifying framework for behavior-driven methods should capture both the pure novelty maximization and pure fitness maximization extremes, as well as a trade-off space, that potentially captures some of the existing approaches and suggests new ones. Many components of existing ranking mechanisms (Section~\ref{SubSecExistingBDAs}) can be represented in terms of pair-wise relationships between solutions, based on their behaviors and fitnesses. These pairwise interactions capture the positive or negative effects solutions have on each other during ranking when they are competing for a spot in the population. Focusing on pairwise effects also helps avoid unintended global effects, such as that discussed in Section~\ref{SubSecSpooky}.

To focus search on maintaining the most efficient set of stepping stones, behavior domination aims to formalize the idea that a solution should dominate solutions with similar behaviors and lower fitnesses. In particular, each solution exerts a domination effect over each weaker solution. Intuitively, the domination effect should increase (decrease) as the difference between their fitnesses increases (decreases), and increase (decrease) as the distance between their behaviors decreases (increases). The following definition of domination effect captures these requirements.

\begin{definition}
The \emph{domination effect} of $x$ on $y$ is a function $$e(x, y) = f(x) - f(y) - d(b(x), b(y))$$ where $f$ is a fitness function, $b$ is a behavior characterization, and $d$ is a behavior distance.
\end{definition}

The score produced by the domination effect function can be used in various ways in a ranking system. Two common methods of combining pairwise scores are (1) ranking by aggregation, and (2) ranking by domination. In ranking by aggregation, solutions are ranked by a single score based on a sum of  pairwise scores, e.g., the novelty score is a normalized sum of distances between the behaviors of pairs of solutions. In ranking by domination, solutions are ranked in a partial order, by a boolean pairwise relation of whether they dominate one another. To enable ranking by domination, the following definition provides such a pairwise operator, based on the domination effect function defined above.

\begin{definition}
\label{DefDomination}
If $e(x, y) \geq 0$, then $x \succeq y$, that is, $x$ \emph{dominates} $y$.
\end{definition}

It turns out that for any specification of effective domination, i.e., any choice of $f$, $b$, and $d$, this definition of domination defines a partial order over solutions.

\begin{theorem}
$y \succeq x$ induces a partial order over solutions for any choice of $f$, $b$, and $d$.
\end{theorem}
\begin{proof}
Transitivity: Suppose $x \succeq y$ and $y \succeq z$. Then, $0 \leq e(x, y) + e(y, z) = (f(x) - f(y) - d(x, y)) + (f(y) - f(z) - d(y, z)) = f(x) - f(z) - (d(x, y) + d(y, z)) \leq f(z) - f(x) - d(x,z) = e(x, z) \implies x \succeq z$. Reflexivity and antisymmetry are similarly straightforward to show.
\end{proof}

The partial order defined by behavior domination is similar to the one defined by Pareto-dominance in multiobjective optimization. Note that, even though they make use of a notion of Pareto-dominance, neither NSGA-NF nor NSLC have the property of a stable partial-ordering of solutions, because the novelty objective fluctuates as the population changes over time. On the other hand, the front induced by behavior domination can be viewed geometrically as a rotation of a Pareto front (Figure~\ref{FigRotation}).
Algorithms based on behavior domination can then more easily inherit properties from multiobjective optimization, e.g., guarantees that the non-dominated front dominates every point ever generated and all area dominated by any point ever generated, and guarantees regarding near-optimal distribution of non-dominated solutions \cite{laumanns02, deb16, coello07}. The practical expectation is that the utility of non-dominated solutions as stepping stones in multiobjective optimization will transfer to the case of behavior domination. An algorithm based on this connection to multiobjective optimization is introduced in Section~\ref{SecNewBDMA}. 

Although aggregation and domination are the most prevalent approaches to ranking, the definition of a behavior domination algorithm does not preclude the existence of other schemes that use a domination effect function.

\begin{definition}
Every algorithm whose ranking mechanism's dependence on $f$ and $b$ can be defined in terms of a domination effect function is a \emph{behavior domination algorithm (BDMA)}.
\end{definition}

Behavior domination algorithms can avoid ``spooky action at a distance'' (Section~\ref{SubSecSpooky}) by using a domination-based ranking scheme. When ranking decisions are only made with respect to the operator $\succeq$, moving a solution $y$ away from a non-dominated solution $x$ cannot cause $x$ to become dominated. For example, see the representation of MAP-elites in the next section (Observation~\ref{ObsMEBDMA}).

\subsection{BDMA Representation of Existing Algorithms} \label{SubSecExistingBDMAs}

The next three observations demonstrate how the behavioral domination framework can be used to represent existing algorithms. Such observations are helpful in clarifying the space of BDMAs.

\begin{observation}[Fitness-based search is a BDMA] \label{ObsFitBDMA}
Since fitness-based search does not make use of behavior, this can be achieved by setting $b$ to be the trivial behavior characterization, $b(x) = 0 \ \forall x$. Then, $\succeq$ (Definition~\ref{DefDomination}) induces the same total ordering as sorting fitness scores directly.
\end{observation}

\begin{observation}[Novelty search is a BDMA]
This is another trivial case. Since novelty search does not make use of the fitness function, this is similarly achieved by choosing $f(x) = 0 \ \forall x$, and using the usual novelty search aggregation scoring for ranking solutions.
\end{observation}

\begin{observation}[MAP-elites is a BDMA] \label{ObsMEBDMA}
Consider an instance of MAP-elites with fitness function $f$, behavior characterization $b_o$, and binning function $\beta$ that maps each behavior to its bin. Choose $b$ such that $b(x) = \beta(b_o(x))$, and define $d$ by
$$
d(b(x), b(y))=
\begin{cases}
0,& \text{if } \ b(x) =  b(y),\\
\infty,& \text{otherwise.}
\end{cases}
$$
Then, the non-dominated solutions under $\succeq$ are exactly the elites maintained by the original MAP-elites algorithm.
\end{observation}

The above subsumptions demonstrate the breadth of the space of BDMAs. However, each of these representations avoids the natural geometric form of the domination effect function. Section~\ref{SecNewBDMA} develops an algorithm that follows more directly from Definition~\ref{DefDomination}.

\subsection{A non-dominated sorting BMDA: BDMA-2} \label{SecNewBDMA}

Given a fitness function $f$ and a behavior characterization $b$, here let the domination effect function be parameterized completely by the choice of behavior distance $d$. 
%Note this is how both fitness-based search and MAP-elites were constructed in Section~\ref{SubSecExistingBDMAs}. 
A new algorithm, BDMA-2, is defined with a scaled L2 distance metric:
$$ d(b(x), b(y))= w \cdot \lVert b(x) - b(y) \rVert_2 .$$
The inclusion of the scaling parameter $w$ is useful for flexibility in relating fitness and behavior distance numerically. Increasing $w$ increases the emphasis on novelty; decreasing it increases the emphasis on fitness. Figure~\ref{FigFourPeaks}
\begin{figure}
\includegraphics[width=0.9\columnwidth]{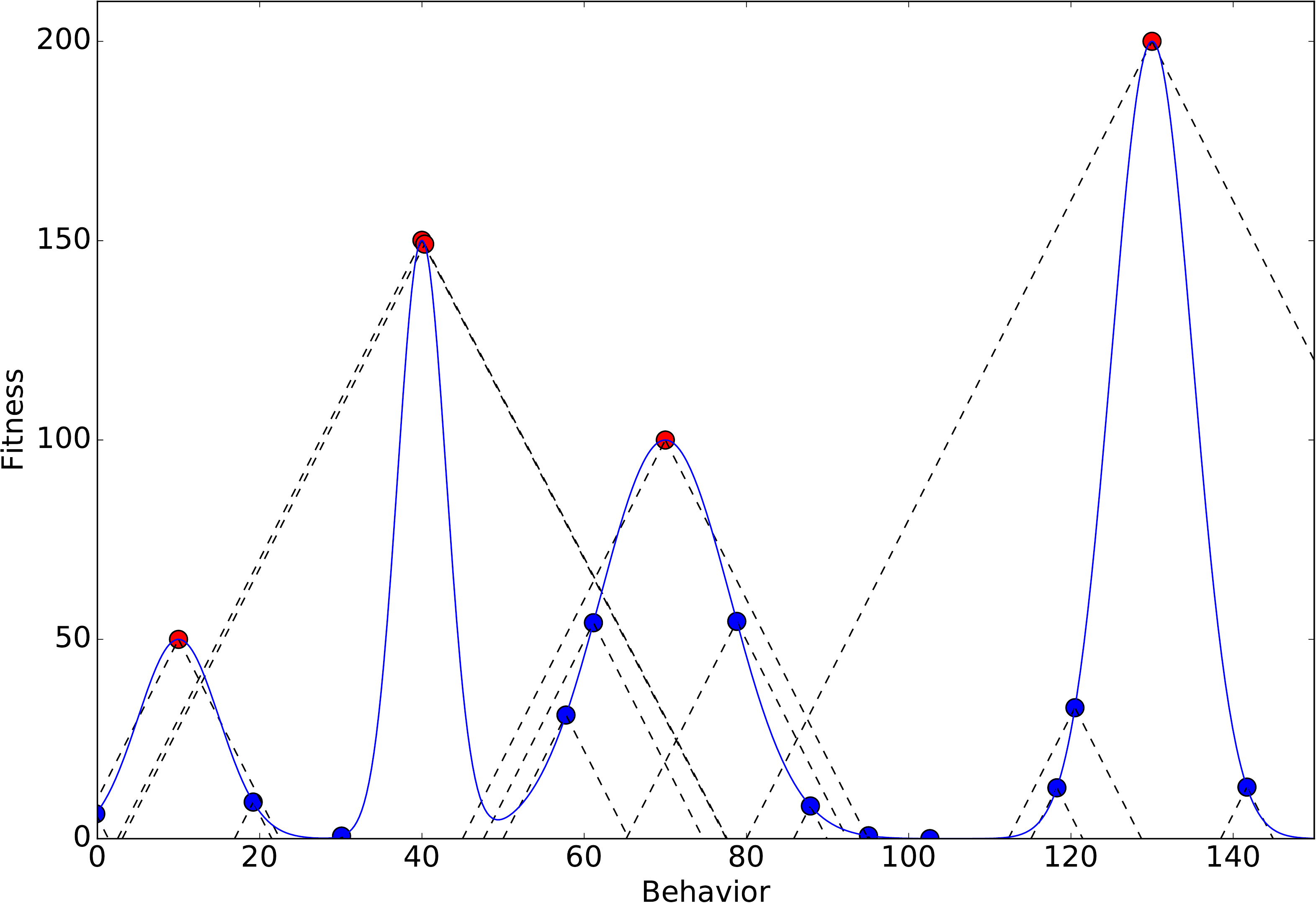}
\hspace{5pt}
\caption{A sample BDMA-2 population successfully maintaining solutions at each local maximum discovered in the four peaks domain (Section~\ref{SubSecMaintaingStones}). Dashed lines indicate the region each solution dominates for $w = 4$. The five solutions on the non-dominated front are in red, including two around the peak where $b(x) = 40$.} 
\label{FigFourPeaks}
\end{figure}
depicts an instance of a ranking step in BDMA-2, including the induced domination structure, taken from the experiments in Section~\ref{SubSecMaintaingStones}. 

Now that a suitable behavior distance is defined, a fast non-dominated sort (as in NSGA-II \cite{deb02}) is used to rank the solutions, based on the $\succeq$ operator induced by $d$.
In contrast to the distance function used by MAP-elites (Obs.~\ref{ObsMEBDMA}), the L2 distance allows the flexible discovery of the locations of an efficient set of stepping stones, opposed to having their bounded locations determined beforehand.
The expectation is that the success of the non-dominated front in NSGA-II in providing useful stepping stone for multiobjective optimization will transfer to this case of behavior domination.
Similar to a previous behavior-driven tie-breaking approach \cite{hodjat16}, ties are broken on the final front from which solutions must be kept by iteratively excluding the less fit of the two nearest solutions on that front, until the desired number of solutions remain. 

Specifying the number of top solutions to select via the fast non-dominated sort can be viewed as specifying the number of stepping stones wished to be maintained during search. 
To preserve the efficient exploration capabilities of novelty search while maintaining useful stepping stones, it is useful to have a subset of the population selected as stepping stones, and the remainder selected by novelty alone. 
Specifying the number of stepping stones in the population is an intuitive parameterization that can be informed by domain knowledge as well as time and space requirements.

On the other hand, it may take significant experimenter effort and domain knowledge to set an effective $w$. Conveniently, the definition of behavior domination can be used to develop a suitable scheme for automatically setting $w$ online during search. It is straightforward to encode rules so that $w$ is set to guarantee the domination or non-domination of some set of solutions considered harmful or desirable, respectively. In the experiments in this paper, an example of such an online adaptation scheme is considered, inspired by the avoidance of ``spooky action at a distance'' (Section~\ref{SubSecSpooky}). In this scheme, at every iteration $w$ is set at the maximal value such that neither of the two most distant solutions are dominated. This online adaptation scheme (BDMA-2a) is compared against setting a static $w$ in Section~\ref{SecExperiments}. Though it is an intuitive heuristic, setting $w$ online in this fashion does not necessarily preserve the guarantees of using a fixed domination effect function. Development of more grounded approaches to adapting $w$ is left to future work.

\section{Experimental Investigation} \label{SecExperiments}

Experiments were run in domains that extend limited capacity drift models, previously used to study novelty search \cite{lehman13, lehman15}, with fitness and a continuous solution space. Each solution is encoded by a vector with values in the range $[0, 150]$. The population is randomly initialized with all values in $[0, 1]$.
This abstraction captures the property of real world domains that often only a small portion of the behavior space can be reached by randomly generated solutions, e.g., robots that either spin in place or crash into the nearest wall; evolution must accumulate structure in its solutions to progress beyond this initial space. 
The first set of experiments tests the ability to discover and maintain available stepping stones; the second tests the ability to perform well in settings where effective use of stepping stones can accelerate evolutionary progress. 

The underlying evolutionary algorithm for each experimental setup is a steady-state algorithm with Gaussian mutation and uniform crossover. The only difference between setups in a domain is the method of ranking solutions. 
See Appendix for experimental parameter settings. 
In each domain, the performance measures for each algorithm were averaged over ten runs.

\subsection{Discovering and Maintaining Stepping Stones} \label{SubSecMaintaingStones}
The first domain has a one-dimensional solution space.
The fitness landscape has four peaks of differing heights, with the rightmost peak being the highest (Figure~\ref{FigFourPeaks}).
The behavior characterization is the identity function, i.e., $b(x) = x$.
Each peak represents a potentially useful stepping stone, with the higher peaks having more potential. In an optimal state, a population will include solutions near the tops of each peak. This domain tests an algorithm's ability to grow its solutions to successfully discover each peak while maintaining in the active population potentially useful stepping stones encountered along the way.

Consider four bins in the behavior space, each of width 10 and centered around a peak. Each algorithm is evaluated against two MAP-elites-based measures \cite{mouret15, pugh15}. The first is the sum of the top fitnesses ever achieved across the bins; this measures an algorithm's ability to discover stepping stones. The second is the sum of the top fitnesses of these bins in the current population; this measures an algorithm's ability to maintain stepping stones. The results are depicted in Figure~\ref{FigOneDResults}. 
\begin{figure}
\centering
\includegraphics[width=0.9\columnwidth]{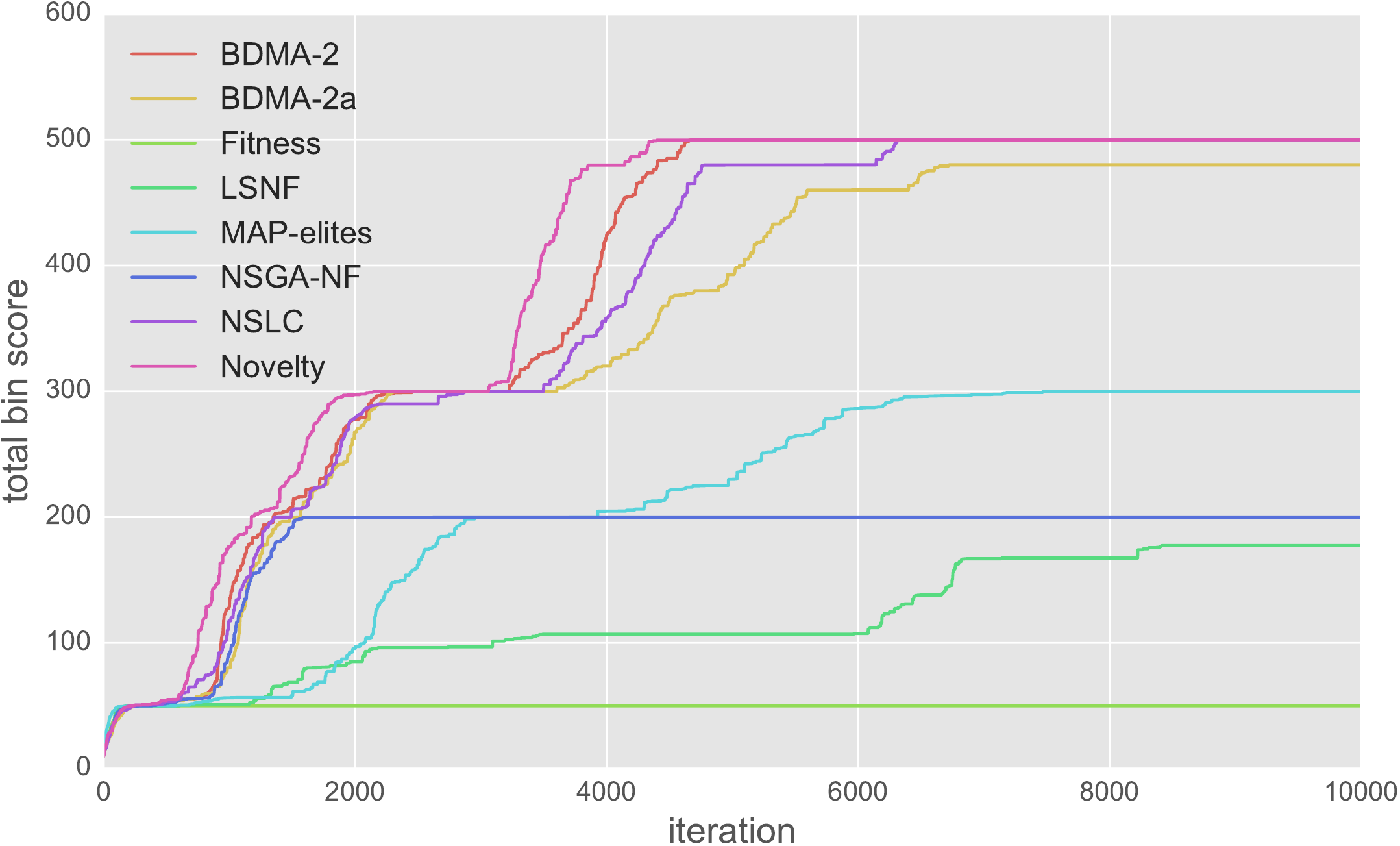}
\includegraphics[width=0.9\columnwidth]{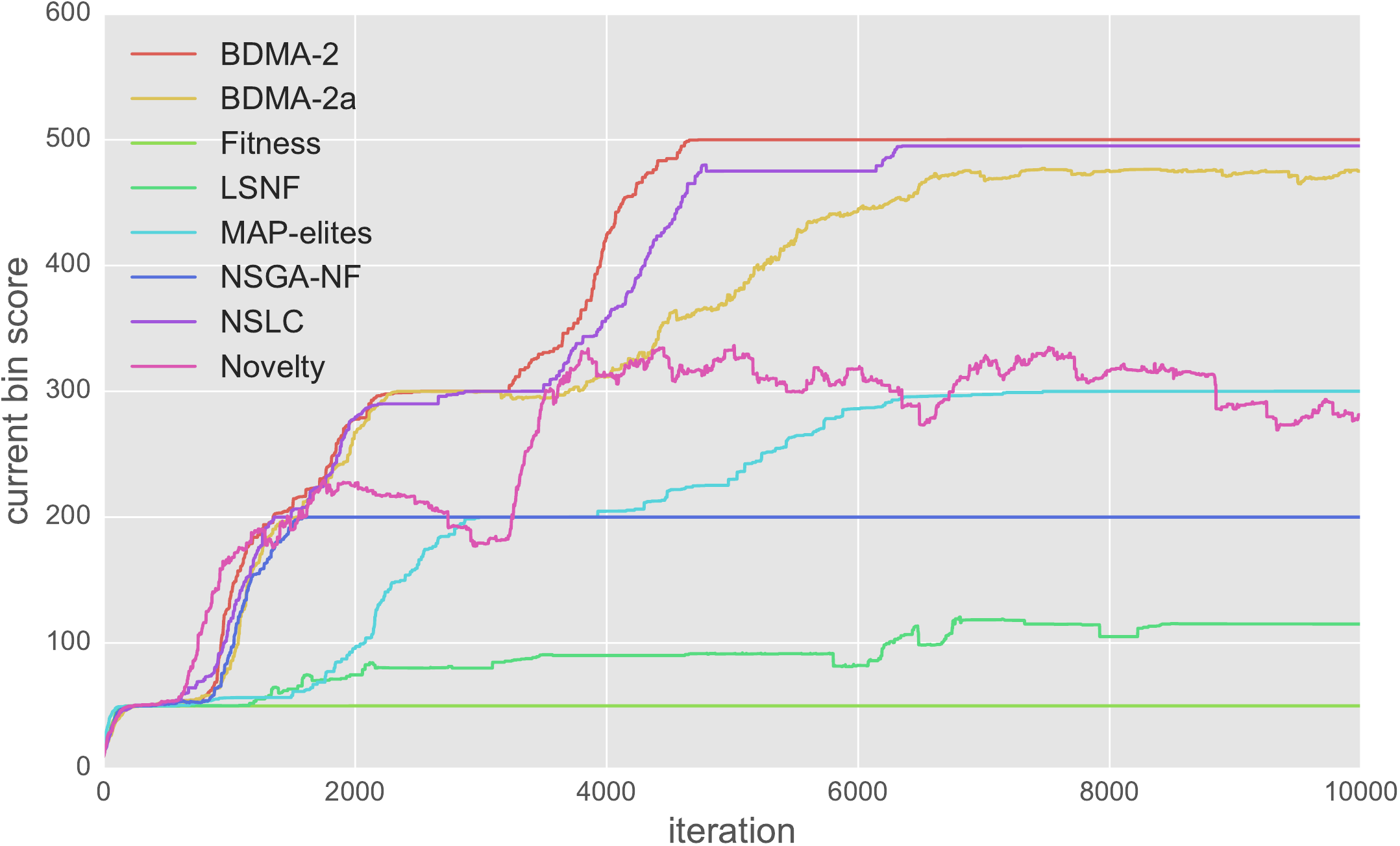}
\caption{Four peaks domain results. A total (current) bin score near 500 indicates all stepping stones are discovered (maintained). (top) Novelty search discovers all the peaks most quickly, but BMDA-2 does not take much longer; (bottom) Only BDMA-2, NSLC, and BDMA-2a consistently maintain solutions near each discovered peak across the ten trials.
\vspace{-10pt}
\label{FigOneDResults}}
\end{figure}
As expected, novelty search is able to discover the available stepping stones most quickly, since it's focused only on exploration. However, BMDA-2 is not far behind, followed by NSLC and BDMA-2a. When it comes to maintaining these stepping stones, BDMA-2 outperforms the other algorithms, again followed closely by NSLC and BDMA-2a. Note that although MAP-elites maintains the elites in each visited bin, when the bin size is large it is difficult to jump to new bins, and when it is small the chance of selecting an elite on the edge as a parent is small. So, MAP-elites explores slowly in this domain (results shown with bin size 1).

Figure~\ref{FigAdaptW} shows examples of values of $w$ adapted over the course of BDMA-2a runs.
\begin{figure}
\includegraphics[width=0.9\columnwidth]{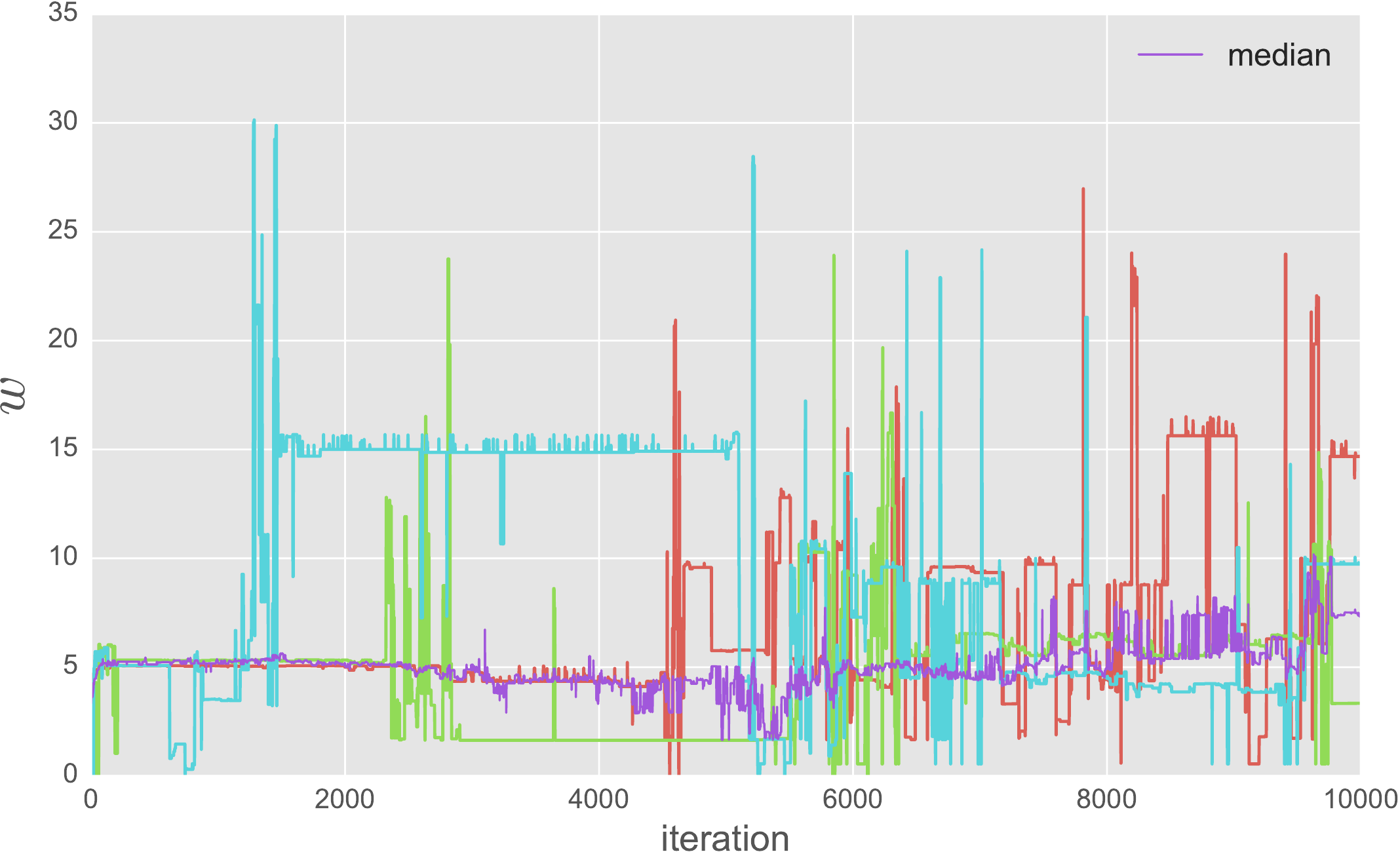}
\caption{Adapted value of $w$ over time for three independent runs of BDMA-2a (Section~\ref{SecNewBDMA}) in the four peaks domain, along with the median $w$ over all 10 runs. Adaptation of $w$ is marked by periods of relative stability followed by periods of relative instability.}
\label{FigAdaptW}
\end{figure}
Future schemes for adapting $w$ may try to minimize fluctuations for better predictability (Section~\ref{SecDiscussionAndFutureWork}).

\subsection{Harnessing Stepping Stones} \label{SubSecHarnessingStones}

The most successful algorithms at discovering and maintaining stepping stones (NSLC, BDMA-2, and BDMA-2a), along with Novelty and Fitness as controls, were evaluated in two further domains, which test the abilities of algorithms to exploit available stepping stones by focusing on the most promising areas of the search space.

\subsubsection{Exponential Focus (ETF) Domain} \label{SubSubSecETF}

The ETF domain captures the notion that real world domains contain complementary stepping stones, which, if harnessed successfully, can accelerate progress in a way not possible otherwise. 
This domain has a two-dimensional solution space, and the fitness function contains stepping stones that can enable exponential progress if used effectively.

The fitness landscape consists of a series of claw-like regions that increase in size and value as they get farther away from the origin; all other areas have fitness zero (Figure~\ref{FigETFResults} (top)).
\begin{figure}
\includegraphics[width=0.9\columnwidth]{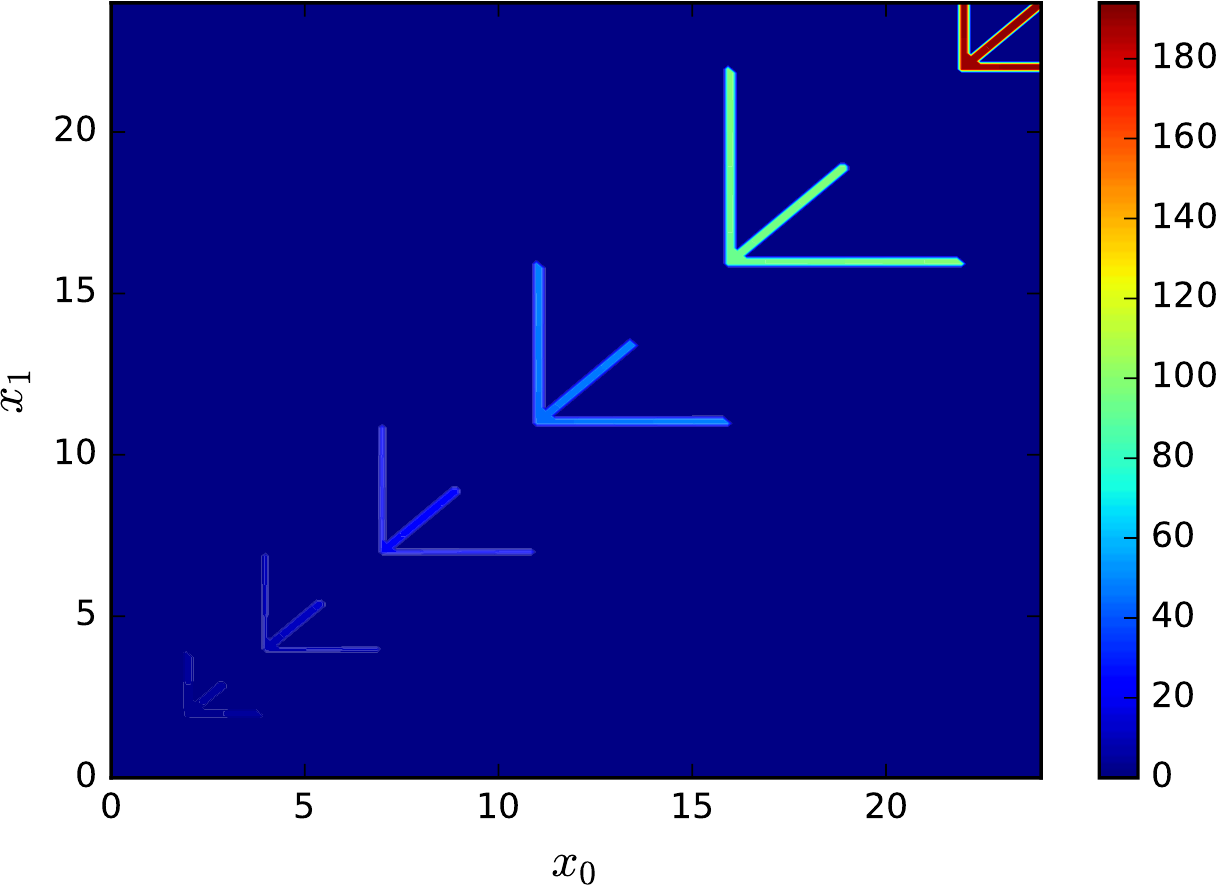}
\\ \vspace{10pt}
\includegraphics[width=0.9\columnwidth]{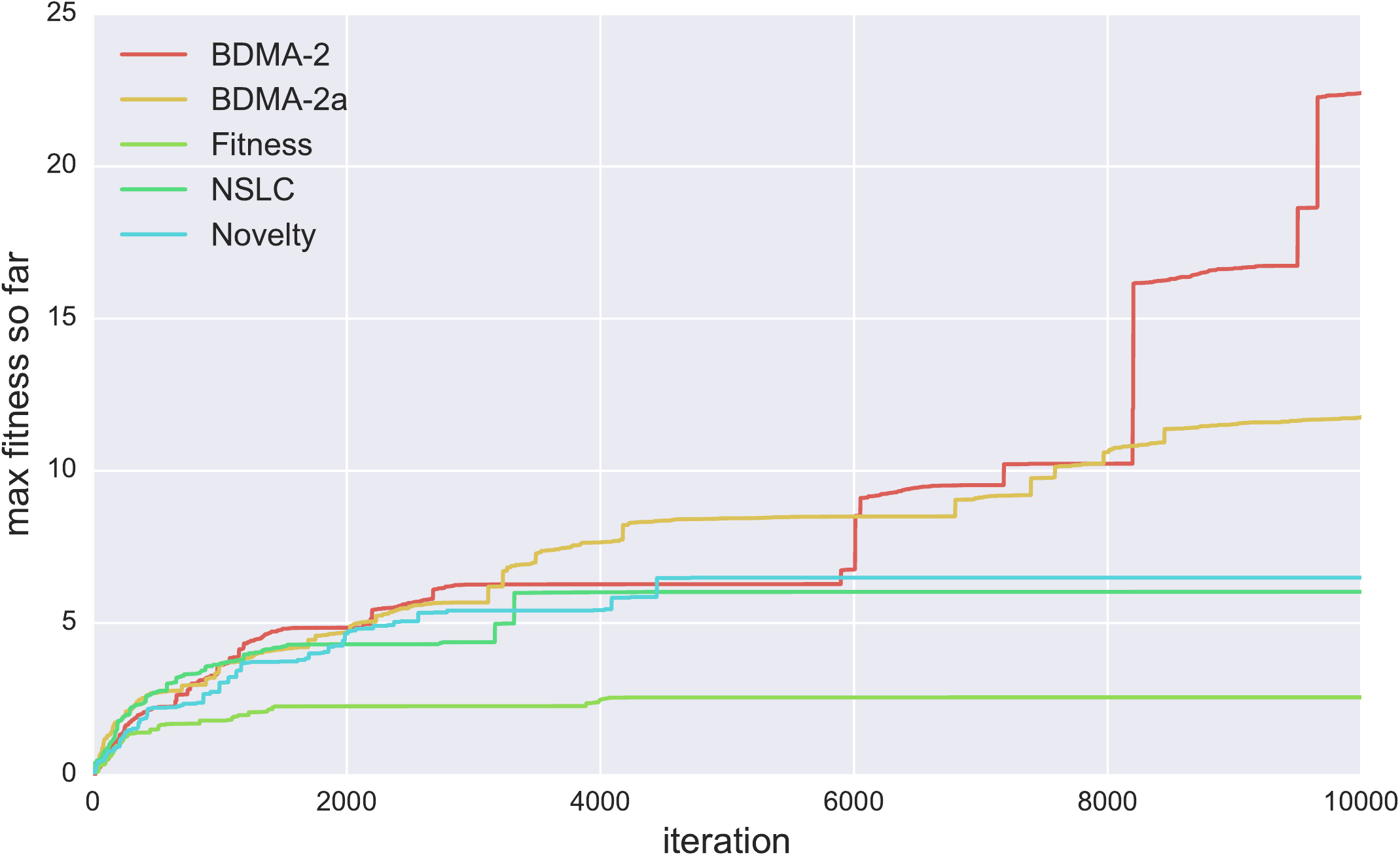}
\caption{(top) The ETF domain contains a series of claw-like regions.
Each region supports two stepping stones that can be combined to reach the next higher-valued region via crossover. 
This domain tests the ability to harness these stepping stones; 
(bottom) Results in the ETF domain with $s = 100$. BDMA-2 is the most successful, followed by BDMA-2a. \label{FigETFResults}}
\end{figure} 
The heel of the first claw is located at $(1, 1)$ and has fitness 1. 
The $i^{th}$ claw has a heel with fitness $h$, and three toes, each of width $\epsilon = 0.2$.
Fitness increases linearly along each toe. 
The tip of the vertical and horizontal toes have fitness $h + i$, and the tip of the diagonal toe has fitness $h + 2i$.
The heel of the $(i + 1)^{st}$ claw has fitness $2(h + i)$, and can be reached by a successful crossover of the $i^{th}$ vertical and horizontal toes. 
Thus, an algorithm can reach the next claw by maintaining solutions on the tips of both horizontal and vertical toes, while avoiding convergence to the deceptive diagonal toe. 

The behavior characterization is $b([x_0, x_1]) = s \cdot x_0 + x_1$, i.e., $s$  controls how much the first dimension of the behavior space is stretched. 
As $s$ increases, it is more costly for an algorithm to densely explore the entire behavior space. 
Experiments were run with $s = 100$, $s = 1000$, and $s = 10000$. 

Since the purpose of this domain is to evaluate how well an algorithm can use stepping stones to discover high-performing solutions, algorithms are compared based on their maximum fitness achieved by iteration. 
Results are shown in Figure~\ref{FigETFResults} (bottom) and Table~\ref{TableResults} (a). 
\begin{table}
\vspace{25pt}
{\footnotesize
\begin{tabular}{| c | c | c | c | c | c |}
\hline
$s$ & Fitness & Novelty & NSLC & BDMA-2 & BDMA-2a \\ \hline 
$100$ & 2.55 (0.28) & 6.49 (0.77) & 6.02 (1.29) & \textbf{22.41} (5.32) & 11.76 (0.85) \\ \hline
$1000$ & 2.55 (0.28) & 9.59 (1.74) & 6.31 (1.26) & \textbf{14.79} (2.63) & 14.16 (1.33) \\ \hline
$10000$ & 2.55 (0.28) & 9.36 (1.68) & 6.13 (0.98) & 9.57 (2.04) & \textbf{15.68} (1.71) \\ \hline
\end{tabular}
}
\\ \vspace{5pt} (a) Mean max fitness (std. err.) in the ETF domain. \vspace{15pt} \\
{\footnotesize
\begin{tabular}{| c | c | c | c | c | c |}
\hline
$D$ & Fitness & Novelty & NSLC & BDMA-2 & BDMA-2a \\ \hline
10 & 2.708 (0.00) & 2.846 (0.09) & 2.823 (0.05) & \textbf{3.023} (0.09) & 3.010 (0.10) \\ \hline
20 & 2.708 (0.00) & 2.678 (0.01) & 2.748 (0.02) & \textbf{2.898} (0.05) & 2.791 (0.05) \\ \hline
30 & 2.708 (0.00) & 2.682 (0.01) & 2.705 (0.00) & \textbf{2.791} (0.02) & 2.711 (0.02) \\ \hline
\end{tabular}
}
\\ \vspace{5pt} (b) Mean max fitness (std. err.) in the focused Ackley domain. \vspace{5pt} \\
\vspace{15pt}
\caption{Max fitnesses achieved through 10,000 iterations, averaged across 10 runs. (a) Results in the ETF domain. Both BDMA-2 and BDMA-2a outperform the other approaches across all scales of $s$. BDMA-2's performance decreases with $s$, while BDMA-2a's increases, showing its ability to successfully adapt $w$ with this type of scaling; (b) Results in the focused Ackley domain. BMDA-2 and BDMA-2a outperform the other algorithms across all scales of $D$.  \label{TableResults}}
\end{table}
BDMA-2a significantly outperforms each existing algorithm for each value of $s$ (Mann Whitney U Test, $p < 0.01$), with BDMA-2 showing dramatic improvements as well.

\subsubsection{Focused Ackley Domain}

The results in the ETF domain demonstrate that BDMA-2 can be successful in domains that contain natural stepping stones.
To further validate this idea, experiments were run in a domain based on the popular Ackley benchmark function \cite{ackley87, back97}, which also has an inherent stepping stone structure. 
The search space is $D$-dimensional. If a solution $x$ falls in a bounded region, defined by $\lvert x_0 - x_1 \rvert < 2$ and $\sum_{i = 2}^D x_i < D / 2$, its fitness is the value of the Ackley function at $[x_0, x_1]$, otherwise, its fitness is drawn randomly from $[0,1]$ (Figure~\ref{FigAckley} (top)).
\begin{figure}
\includegraphics[width=0.9\columnwidth]{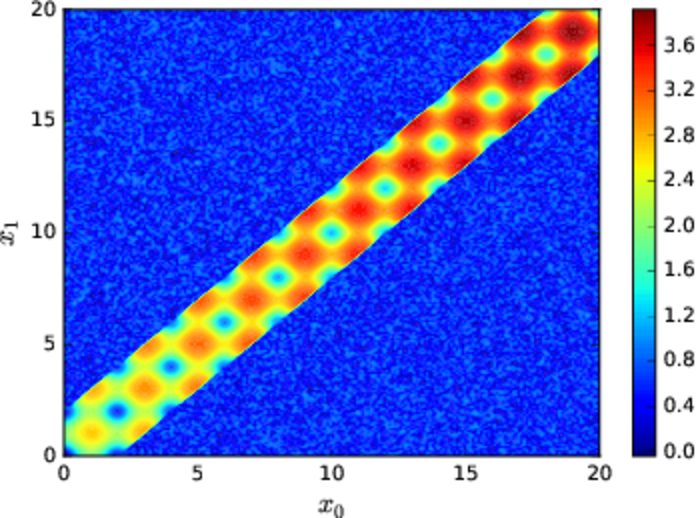}
\\ \vspace{10pt}
\includegraphics[width=0.9\columnwidth]{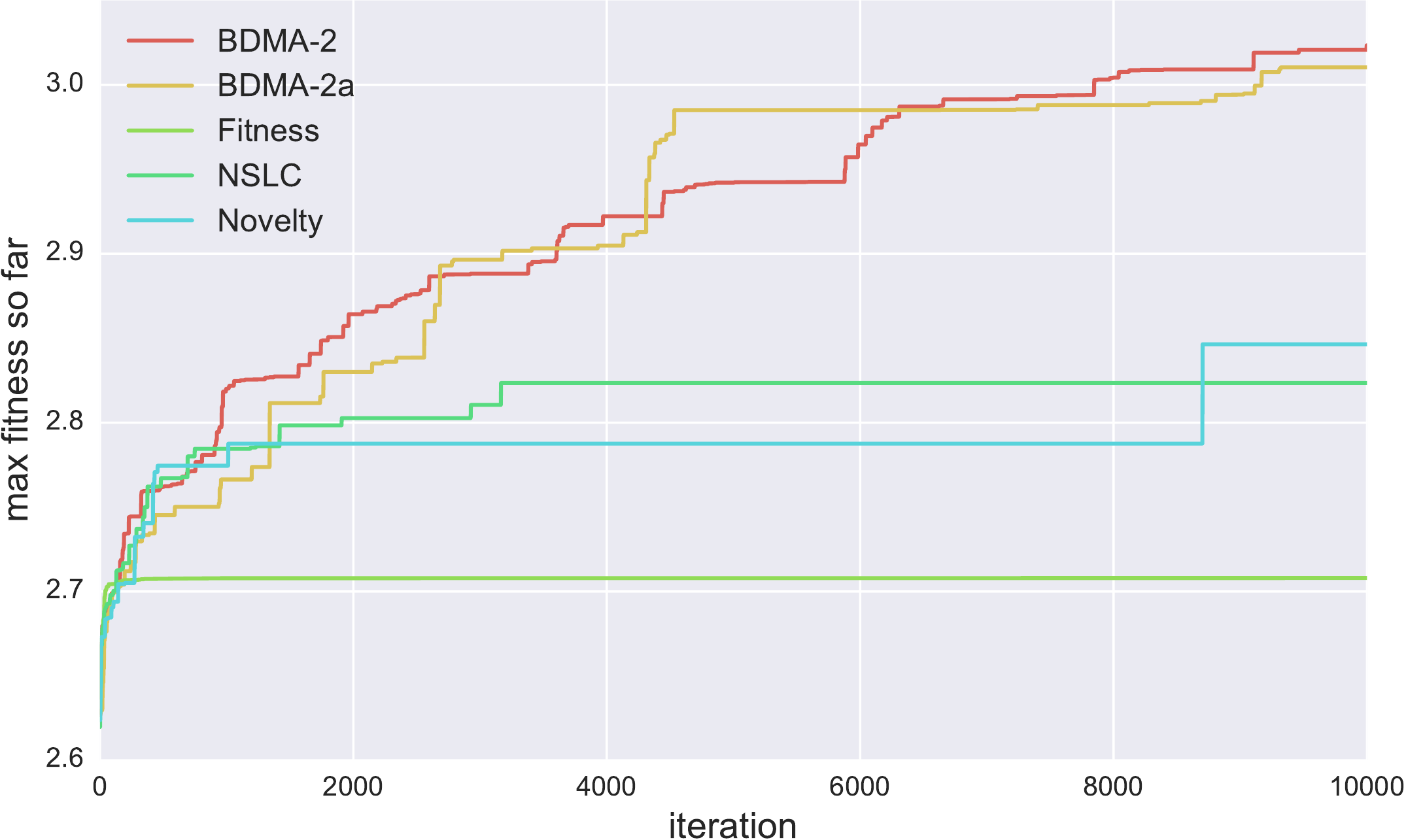}
\caption{
(top) The focused Ackley domain tests an algorithm's ability to focus on useful stepping stones, which here are local maxima bordering noisy regions in a high-dimensional behavior space.
(bottom) Results in the focused Ackley domain with $D = 10$. BDMA-2 and BDMA-2a consistently outperform the other approaches.}
\label{FigAckley}
\end{figure}
In this domain, $b(x) = x$, and scale is controlled by the number of dimensions $D$ of the behavior space.
The noise outside of the bounded region is a challenge for algorithms that must decide which regions are worth exploring.
The results in Figure~\ref{FigAckley} (bottom) and Table~\ref{TableResults} (b) show how BDMA-2 and BMDA-2a improve upon existing approaches. BDMA-2 significantly outperforms each existing algorithm for each value of $D$ (Mann Whitney U Test, $p < 0.02$), except for Fitness with $D = 30$, as each approach that makes use of the behavior characterization $b$ is negatively affected by increases in the dimensionality of $b$. 

Still, the success of BDMA-2 and BDMA-2a in these domains that contain useful stepping stones is encouraging evidence for the potential to scale behavior domination algorithms to more complex domains, where it is assumed that such stepping stones exist.

\section{Discussion and Future Work} \label{SecDiscussionAndFutureWork}

The existing algorithms classified under behavior domination (Section~\ref{SubSecExistingBDMAs}) have been validated across an array of complex domains \cite{back97, lehman10b, lehman11a, mouret15, mouret15a}. 
The experiments in Section~\ref{SecExperiments} demonstrate that the behavior domination framework can lead to progress over existing approaches on problems that contain useful stepping stones, and it will be interesting to see what new methods will be required to scale these methods to the real world, where stepping stones abound, e.g., in domains such as robot control \cite{lehman11a, mouret15, mouret12} and automatic content generation \cite{lehman11b, nguyen16, lehman16}.

Effective specification of behavior is still an issue. Experiments in the ETF domain (Section~\ref{SubSubSecETF}) showed how behavior-driven algorithms can be sensitive even to linear scaling of the behavior space. Although BDMA-2 and BDMA-2a outperformed the other approaches in this scenario, their reliance on a single parameter $w$ across all behavior dimensions makes them susceptible to such issues. From the perspective of behavior domination, solutions to these issues can be hidden in the behavior characterization, i.e., by letting $b$ be some transformation of the raw behavior characterization. Automatically specifying behavior characterizations in a robust and general way is an open problem, and some recent work has begun to make progress in this direction \cite{meyerson16, nguyen16, liapis13, gomes14}.

Given a reasonable behavior characterization, one method of setting $w$ automatically was presented in Section~\ref{SecNewBDMA}, but there are many methods that could be tried, some of which may be more generally effective, and preserve stability properties of the behavior domination front. Overall, more work can be done to transfer guarantees from the theory of multiobjective optimization \cite{deb16, coello07}, which will also lead to practical algorithmic improvements.

Although transferring theoretical properties can be satisfying, further work is needed to understand where theoretical focus in behavior-driven search will yield the biggest practical impact. The issue of ``spooky action at a distance'' (Section~\ref{SubSecSpooky}) identifies some unsettling dynamics in existing algorithms, but it is not clear whether it strikes at the heart of the matter, or is merely a shadow of something more illusive. Further work must be done to fully characterize the emergent dynamics of ranking procedures, in parallel with work to understand how careful specification of a behavior characterization and fitness function can guarantee the existence of useful stepping stones in the joint behavior-fitness space.

\section{Conclusion}

The goal of this study was to understand and harness the ability of evolution to discover useful stepping stones.
Existing behavior-driven algorithms have properties that interfere with this goal; the behavior domination framework was introduced to reason formally about how these properties could be avoided.
A new algorithm, BDMA-2, was introduced based on this framework, and shown to improve over existing behavior-driven algorithms in domains that contain useful stepping stones.
The behavior domination perspective is thus a promising tool for comparing and understanding existing behavior-driven algorithms as well as for designing better ones in the future.

\bibliographystyle{ACM-Reference-Format}
\bibliography{meyerson}

\appendix
\vspace{10pt}
\noindent\textbf{Appendix of Experimental Parameters}
\vspace{2pt}
\footnotesize

\vspace{2pt}\noindent\emph{BDMA-2 params.:}
$w$ by domain: Four peaks: $w = 16$; ETF and Ackley: $w = 0.005$, $w = 0.0005$, $w = 0.00005$, for $s (D) = 100 (10), 1000 (20), 10000 (30)$, resp. Proportion of population selected by novelty alone: 0.5.

\vspace{2pt}\noindent\emph{Underlying algorithm:} 
pop. size = 20; 
solutions generated per iteration = 1; 
crossover probability = 1; 
mutation $\sigma$ by domain: $\sigma = 1$, $\sigma = 0.1$, $\sigma = 0.25$.

\vspace{2pt}\noindent\emph{Novelty params.:}
$k = 5$;
$p_{add} = 0.01$ (BDMA-2 has no external archive).

\vspace{2pt}\noindent\emph{Four peaks fitness function:} $f(x) = 50 \cdot g(x, 10,5) + 150 \cdot g(x, 40,3) + 100 \cdot g(x, 70, 8) + 200 \cdot g(x, 130,5)$, where $g(x, \mu, \sigma) = \text{exp}[-(x - \mu)^2/(2\sigma^2)]$.

\vspace{2pt}\noindent\emph{Ackley function parameterization:} $a = 500$, $b = 0.0005$, $c = \pi$.

\end{document}